%% file: neurips_2023.tex
\documentclass{article}

    \usepackage[preprint]{neurips_2023}

\usepackage[utf8]{inputenc} 
\usepackage[T1]{fontenc}    
\usepackage{hyperref}       
\usepackage{url}            
\usepackage{booktabs}       
\usepackage{amsfonts}       
\usepackage{nicefrac}       
\usepackage{microtype}      
\usepackage{xcolor}   
\usepackage{graphicx}
\usepackage{amsmath}
\usepackage{amsthm}
\usepackage{amssymb}

\newtheorem{theorem}{Theorem}

\title{Training Reinforcement Learning Agents and Humans With Difficulty-Conditioned Generators}

\author{%
  Sidney Tio \\
  \texttt{sidney.tio.2021@phdcs.smu.edu.sg} \\
  \And
   Jimmy Ho \\
  \texttt{jimmyho@smu.edu.sg} \\
  \And
  Pradeep Varakantham \\
  School of Computing and Information Systems\\
  Singapore Management University\\
  \texttt{pradeepv@smu.edu.sg} \\
}

\begin{document}

\maketitle

\begin{abstract}
  We introduce Parameterized Environment Response Model (PERM), a method for training both Reinforcement Learning (RL) Agents and human learners in parameterized environments by directly modeling difficulty and ability. Inspired by Item Response Theory (IRT), PERM aligns environment difficulty with individual ability, creating a Zone of Proximal Development-based curriculum. Remarkably, PERM operates without real-time RL updates and allows for offline training, ensuring its adaptability across diverse students. We present a two-stage training process that capitalizes on PERM's adaptability, and demonstrate its effectiveness in training RL agents and humans in an empirical study.
\end{abstract}
\begin{figure}[!h]      
  \centering
  \includegraphics[keepaspectratio=true,scale=0.25]{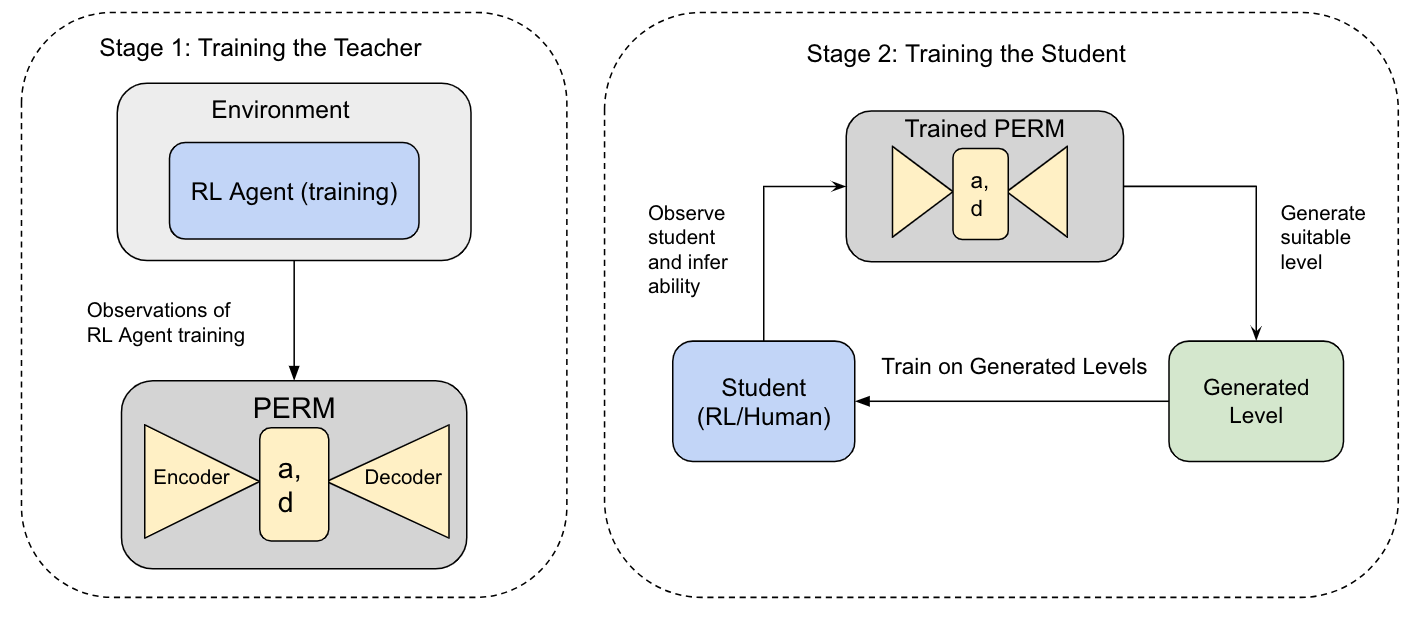}
  \caption{Overview of the proposed 2-stage process. In Stage 1, the IRT-based Parameterized Environment Response Model (PERM) observes a Reinforcement Learning (RL) Agent as it trains in a given environment with randomized levels. During this stage, PERM learns to accurately infer both student ability and level difficulty. In Stage 2, once trained, PERM is deployed to train both artificial and human students. It achieves this by inferring their current ability and providing suitable training levels within the same domain.}
  \label{fig:teaser}
\end{figure}
\section{Introduction}
\textit{Unsupervised Environment Design} (UED, \cite{dennis2020emergent}) has emerged as a promising approach for generating adaptive curricula in a teacher-student paradigm. By identifying useful environments that optimize student learning and taking student performance as feedback, UED tackles the challenge of creating a curriculum that strikes the optimal balance between challenge and current ability, as known to occur in the Zone of Proximal Development (ZPD). However, previous UED (e.g. \cite{parker2022evolving}, \cite{du2022takes}) methods have relied on surrogate objectives or co-learning with other agents, which may not directly address the ZPD. 

\cite{tio2023transferable} proposed using Item Response Theory (IRT, \cite{embretson2013item}) to address this problem. IRT is a mathematical framework used in standardized testing and the design, analysis, and scoring of tests and questionnaires (\cite{hartshorne2018critical},\cite{harlen2001assessment},\cite{luniewska2016ratings}). It allows educators to quantify the ``difficulty" of a test item by modeling the relationship between a test taker's response to the item and their overall ability. In the context of UED, IRT provides a useful framework to understand the difficulty of a parameterized environment relative to the student's ability, which we aim to maximize.

Our work introduces and expands on \textit{Paramterized Environment Response Model}, or PERM \citep{tio2023transferable}. PERM applies the IRT to the UED context and generates curricula by matching environments to the ability of the student, hence providing challenges that are constantly within ZPD. We also demonstrate that this approach is transferable across students, including human students.  

\noindent \textbf{Contributions:} To summarize, our contributions are as follows:
\begin{enumerate}
    \item We introduce PERM \citep{tio2023transferable}, a generative model based on IRT, and demonstrate how it can be applied across different learning context by adapting it to our 2D-game.    
    
    \item We introduce a two-stage process to first exploit RL to collect data for PERM, and secondly deploy PERM to train learners, including real world agents. 
    
    \item Through empirical studies and human subject experiments, we demonstrate the effectiveness of PERM as a teacher algorithm for both RL student agents, and human learners. To our knowledge, this approach is the first algorithm that has demonstrated transferrability in training both artificial agents and real humans. 
\end{enumerate}

\section{Related Works}
\subsection{Item Response Theory}
In Psychology and Education, IRT is employed to model the relationship between a test taker's ability and a specific question characteristic, typically its difficulty. The primary aim is to assess a student's ability based on their responses to questions of varying difficulty levels. In this paper, we focus on the continuous variant of the 1-Parameter Logistic (1PL; \cite{rasch1993probabilistic}) IRT model. In RL settings, we adopt the continuous IRT model to characterize the interactions between an agent and its environment, which are summarized by the cumulative rewards achieved. The continuous 1PL model is formally defined as follows::
\[ 
    p(Z \leq r_{i,j} | a_i, d_j) = \frac{1}{\sqrt{2\pi}} \int_{-\infty}^{a_{i} - d_{j}}\exp\{-\frac{u^2}{2}\}du
\]
\label{eq:ogive}
where $r_{i,j}$ is the response by the $i$-th person, with an ability measure $a_i$, to the $j$-th item, with a difficulty measure $d_j$. We see that this model is equivalent to the cumulative distribution of a standard normal distribution. Therefore, the probability that a student gets an average score on the item is a function of the difference between student ability ${a_i}$ and item difficulty $d_j$.

While earlier works have used different methods to perform inference for IRT, a recent method, VIBO \citep{wu2020variational}, introduces a variational inference approach to estimate IRT. More critically, the formulation of the IRT as a variational inference problem allows us to exploit the learned representation to generate new items.

\subsection{Zone of Proximal Development}
Prior work in UED discusses the ZPD \citep{vygotsky1978mind}, loosely defined as problems faced by the student that are not too easy (such that there is no learning value for the student) and not too difficult (such that it is impossible for the student). 

PAIRED \citep{dennis2020emergent} introduces an adversarial teacher that generates environments to maximize the regret between a protagonist student and an antagonist agent. ZPD is incorporated into the adversarial framework by motivating the teacher agent to provide tasks that are not too challenging to make the antagonist fail but not so trivial that the protagonist achieves high rewards. We note that operationalizing ZPD here requires an additional agent. 

PLR \citep{jiang2021prioritized}, and its newer variants (\cite{parker2022evolving}, \cite{jiang2021replay}), maintains a store of previously seen levels and prioritizes the replay levels where the positive value loss is large. These methods rely on the dissonance between predicted value versus the actual rewards obtained to identify levels with high learning potential. For environments that are outside the ZPD, PLR posits that positive value loss would be small. 

PERM \citep{tio2023transferable} is developed based on the IRT and is augmented with generative capabilities, enabling it to create new items with desired difficulty levels. Thus far, while PERM has demonstrated capabilities in training RL agents in OpenAI's Gym \citep{brockman2016openai} environments, it has yet to replicate the same efficacy in training humans. We provide a brief overview of PERM in the Methods section.

In summary, teacher-student curriculum generation approaches have predominantly focused on ZPD in some aspect, but have relied on surrogate objectives without directly measuring environment difficulty or student ability. PERM, on the other hand, addresses these problems by directly inferring these features, making it a suitable candidate for transfer between students and for training for real world agents. 

\section{Method}
In this section, we highlight the key components of our training framework that utilizes PERM, and how we adopt it to train virtually any student in any environment. 

\subsection{Preliminaries}

For the entirety of this work, we present our work in the context of a proprietary 2D obstacle course game, named Jumper, created with Unity \citep{juliani2020}. The goal of the player is to control a character with their computer keyboards to navigate cross spiked pathways to reach the end goal. If players were to jump into the spikes, the player would restart from the beginning. A sample level could be seen in Figure \ref{fig:level}. 

Under the UED framework, teacher algorithms interact with the student by selecting the environment parameters used to generate the training levels. For Jumper, the parameters are \textit{spike density} - the number of spikes in a given level, and \textit{height variance} - the degree to which the terrain varies. Further details of the environment can be found in the Appendix \ref{app:env}. 

\subsection{Teacher Model Development}
PERM draw parallels from UED to IRT by characterizing each environment parameter $\lambda$ as an item which the student agent with a policy $\pi_t$ 'responds' to by interacting and maximizes its own reward $r$. Specifically, each student interaction with the parameterized environment yields a tuple $(\pi_t,\lambda_t, r_t)$, where $\pi_t$ represents the student policy at $t$-th interaction, and achieves reward $r_t$ during its interaction with the environment parameterized by $\lambda_t$. We then use a history of such interactions to learn latent representations of student ability $a \in \mathbb{R}^n$ and item difficulty $d \in \mathbb{R}^n$, where $a \propto r$ and $d \propto \frac{1}{r}$. In this formulation, $\pi_t$ at different timesteps are seen as students independent of each other.

\subsubsection{Learning Latent Representations of Ability and Difficulty}
PERM uses a Variational Inference problem \citep{kingma2013auto} formulation to learn latent representation of any student interaction with the environment. More critically, PERM exploits the amortization of the item and student space, which allows it to scale from discrete observations of items, to a continuous parameter space such as the environment parameters in UED. From here, we drop the subscript for the notations $a$, $d$, and $r$ to indicate our move away from discretized items and students, as originally formulated in IRT.

\subsubsection{Generating New Levels for Curricula}
The objective of VIBO is to learn the latent representation of student ability and difficulty of items. In order for us to generate the next set of environment parameters $\lambda_{t+1}$for the student to train on, PERM include an additional decoder to generate $\lambda$ given a desired difficulty estimate $d$.  

PERM makes a core assumption that optimal learning takes place when the difficulty of the environment matches the ability of the student. In the continuous response model given in Eq. \ref{eq:ogive}, we see that when ability and difficulty is matched (i.e. $a = d$), the probability which the student achieves a normalized average score $r = 0$ is 0.5. This is a useful property to operationalize ZPD, as we can see that the model estimates an equal probability of the student overperforming or underperforming. 

Training is initialized by uniformly sampling across the range of environment parameters. After each interaction between the student and the environment, PERM estimates the ability $a_t$ of the student given the recent episodic rewards and parameters of the environment. PERM then generates the parameters of the next environment $\lambda_{t+1} \sim p_\theta(\lambda|d_{t+1})$ where $d_{t+1} = a_t$. \\

We refer motivated readers to the Appendix and  \cite{tio2023transferable} for more details. 

\subsection{Deploying PERM as a Training System: A Two-Stage Process}
Equipped with a understanding of PERM's role and capabilities, we now describe a two-stage process which we propose as a generalized framework for training. The process can be visualized in Figure \ref{fig:teaser}.

\subsubsection{Stage 1: Collecting Data to Train PERM}\label{sec:data_collection}
One of the primary challenges facing IRT-based solutions lies in the availability of item-student interaction data. While established standardized tests possess abundant data due to its rich history and wide student base \citep{harlen2001assessment}, applying IRT to new domains remains problematic. Commonly referred to as the cold-start problem \citep{bassen2020reinforcement}, such models often provide a sub-optimal experiences during the infant stages, potentially leading to poor user experience. 

In contrast, training RL agents in simulated environments provides a substantial advantage by offering access to a virtually limitless pool of item-student interaction data. This advantage enables us to address these challenges effectively. 

In the first stage of our approach, we initialize new RL agents to gather the necessary item-student interaction data in a given environment. The RL Agents would be exposed to a variety of levels, all of which were generated using Domain Randomization \citep{tobin2017domain}. We then collect their performance results and the associated level parameters used to generate the levels. 

After the RL agent has completed its training, we train PERM with the entire history of performance-environment parameters to yield a model that is capable of inferring student ability and level difficulty.

\subsubsection{Stage 2: Deploying PERM to Train Others}
In this stage, we implement PERM as the teacher algorithm, in accordance with the UED framework.

During this phase, the pre-trained PERM actively observes the student's interactions with the learning environment, allowing it to make accurate assessments of the student's current ability level. Based on this assessment, PERM then suggests the next set of environment parameters that are deemed appropriate for the student's current stage of learning.

Importantly, since PERM was pre-trained prior to the initiation of the student's training, we effectively circumvent the cold-start problem. It's noteworthy that even when human students commence their training with varying levels of competency, PERM can quickly infer their abilities using the initial training interactions and subsequently offer appropriately challenging content tailored to each student.

\begin{figure*}[!ht]
    \includegraphics[width=\linewidth, height=5cm]{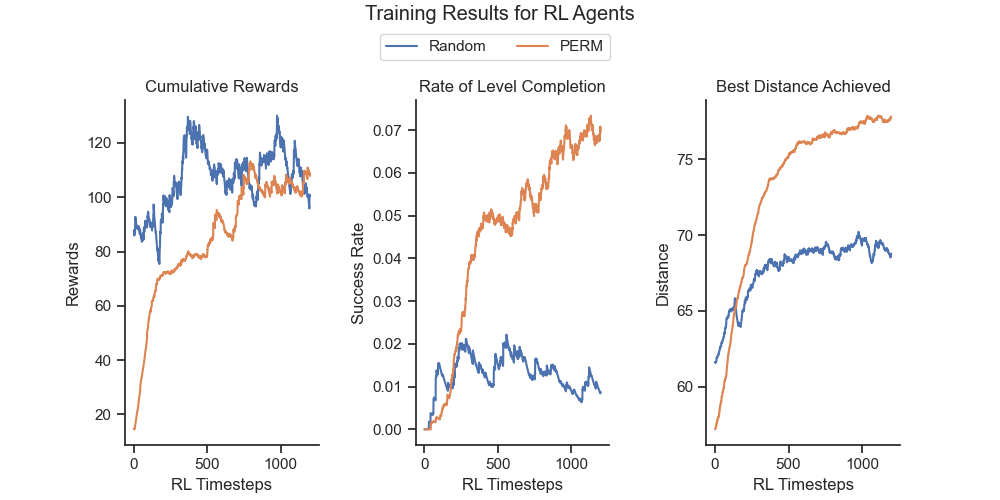}
    \caption{Training results of RL Agents trained under PERM (orange) and a random curricula (blue). Left: Agents trained under PERM achieved comparable training rewards, despite training under more difficult levels. Middle: Agents trained under PERM were more likely to complete the level and reach the final goal during training. Right: Agents trained under PERM travelled deeper into the level than the counterparts in the random condition.}
    \label{fig:rl_analysis}
\end{figure*}
\subsection{Evaluation with RL Students}
\label{sec:pre-study}
To determine if PERM applies well to our Jumper environment, we conducted a study in which we use PERM to train a student RL agent.  

We first train a Jumper-tuned version of PERM with the method highlighted in \hyperref[sec:data_collection]{Stage 1}. In this development phase, we obtained a total of 14506 environment-student interaction data. Thereafter, we deploy the trained PERM as a teacher algorithm to a new PPO \citep{schulman2017proximal} RL student trained using Unity's \texttt{ml-agents} package \citep{juliani2020}. We compare the results of a RL student trained under PERM with one under a random curricula. The results are shown in Figure \ref{fig:rl_analysis}. 

Based on the obtained results, it is evident that the adoption of an IRT-driven curriculum with the PERM teacher yields remarkable outcomes for RL agents, surpassing the performance achieved by the random curriculum. Notably, RL agents trained using the IRT-driven curriculum exhibit a higher level of proficiency in completing levels and, on average, traversed deeper into these levels compared to their counterparts trained using the random curriculum. These outcomes are noteworthy considering that PERM continually challenges the student by evolving the levels in the same pace. 

\section{Evaluation with Human Subjects}
In this section, we evaluate the effectiveness of a curricula designed by PERM and compare it against a randomly designed curricula and a control condition for human students. In this design, participants would be undergoing training to play the Jumper game. 

\subsection{Experimental Design}
Participants were recruited via an invitation posted to a large online chat group established to connect researchers with a large subject pool. In order to mitigate prior gaming experiences from confounding the performance results, we queried participants on their familiarity with 2D side-scrolling\footnote{Side-scrolling is a common term used in video games to describe games in which the camera view follows the character as it traverses from left to right. In our question, we gave \textit{Super Mario Bros} as an example.} games using a Likert scale and balanced the condition assignment accordingly. 

First, participants were introduced to Jumper, and given visual instructions on how to interact with it. Participants were then given a trial to familiarize themselves with Jumper keyboard controls. Once the trial was completed, participants' training would commence. After training, participants would be tested on an unseen level. Participants were compensated upon successful completion of the exit questionnaires and training. 

Prior to the training, participants were randomly assigned to 1 of 3 conditions that would determine the type of training they received: \textit{No Training} - participants in this control condition did not receive any training and immediately went to the test stage after the trial; \textit{Random Curriculum} - participants would receive randomly generated training levels; \textit{PERM} - we deploy the same model used in the \hyperref[sec:pre-study]{RL study} to judge participants on their performance on training levels, and generate the next level according to the inferred ability level of the student. The first level assumes that the participant is of average ability and generates the first level accordingly. For participants in the Random and PERM group, they would have received 10 different levels, with a maximum of 15 attempts per level. Upon reaching the goal, or maximum attempts, the level would terminate and the next training level is generated. Similarly, for the trial and the final test, participants were given up to 15 attempts. 

Further details on the participants, and the procedures of the training can be found in the Appendix. 

\section{Results}
To establish the efficacy of a PERM-driven curriculum, we analyzed the training trial with the following research questions in mind: 
\begin{description}
    \item [R1:] How does an adaptive curriculum generated by PERM affect the completion rate and performance on the final test?
    \item [R2:] What were the differences between the levels generated by the random curricula and PERM?
\end{description}

For all statistical tests described, we used $\alpha = 0.05$ and corrected for multiple comparisons with the Bonferroni correction.

\subsection{R1: Training Outcomes}
The first analysis seeks to investigate how PERM-guided training affects the performance on the final test, and the completion rate of the test. 

\subsubsection{Method}
\begin{figure*}[!t]
    \includegraphics[width=\linewidth,height=5cm]{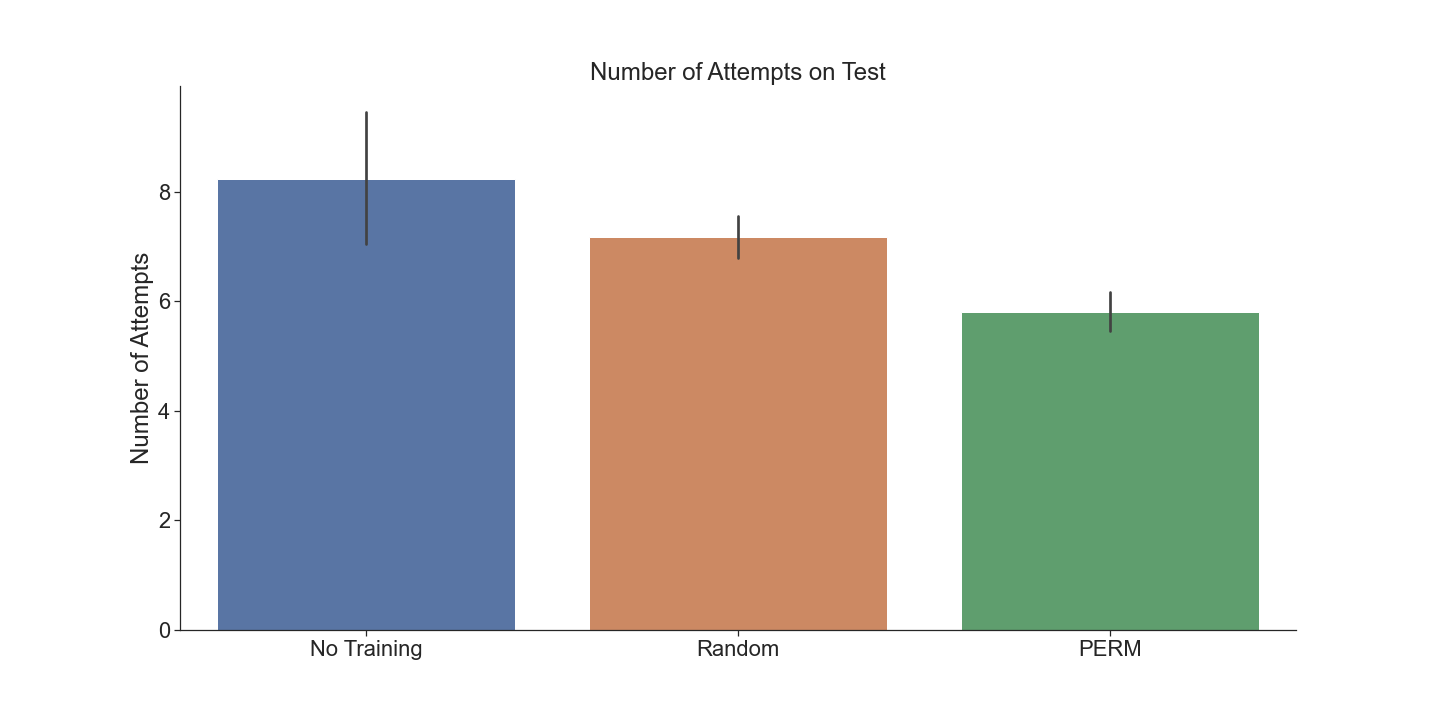}
    \caption{Participants trained by PERM completed the test with lesser number of attempts, compared to the other conditions ($p < 0.01$)}
    \label{fig:511_1}
\end{figure*}

We first compared the number of attempts required to complete the final test. Next, we compared their completion rate. We also compared participant's self-reported familiarity with side-scrolling games against their completion rates. A successful completion meant that participants took lesser than 15 attempts on the final test. Lastly, we analyzed the duration it took per attempt for them to complete. We perform the above analysis based on the assumption that more competent participants would complete the test with lesser attempts, with a shorter duration. We used Student's t-test to compare the duration and the attempts made in the final test, and chi-squared test of goodness of fit to compare completion rates.  
\begin{figure*}[!t]
    \centering    \includegraphics[keepaspectratio=true,scale=0.35]{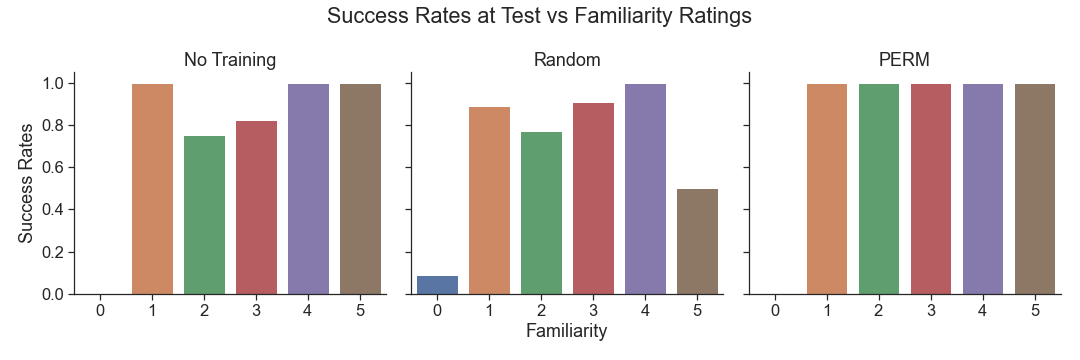}
    \caption{Participant's self-report of their familiarity with 2D games, against their completion rates in the final test. 0 represented "No Experience at all" while 5 represented "Highly Experienced".}
    \label{fig:511_2}
\end{figure*}

\subsubsection{Results}
\begin{figure*}[!t]
    \includegraphics[keepaspectratio=true,scale=0.25]{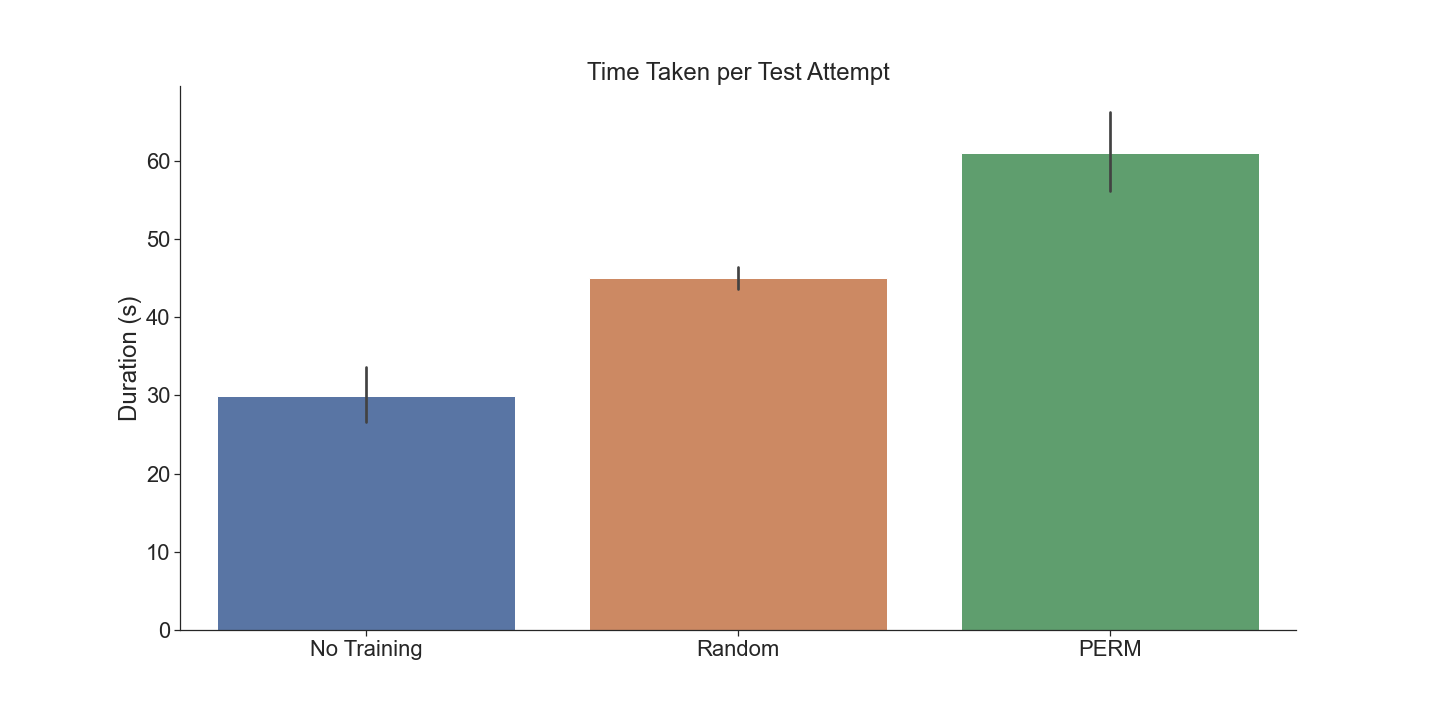}
    \caption{Participants under PERM took a longer time per attempt during the test ($p < 0.01$).}
    \label{fig:511_3}
\end{figure*}
The number of attempts and the completion rate of the tests are presented in Figure \ref{fig:511_1} and \ref{fig:511_2}  respectively. As expected, participants who received any form of training (Random or PERM) performed better than the control group (No training). More critically, students trained under PERM ($\mu = 5.8, \sigma = 4,86$) required significantly lower number of attempts to complete the test, as compared to the no training ($\mu = 8.225, \sigma = 5.46, p < 0.01$) and the random curriculum ($\mu = 7.17, \sigma = 5.49, p < 0.01$).  
Additionally, participants under the PERM were more likely to complete the test (i.e. reach the goal with less than 15 attempts), regardless of prior experience with games, than the other conditions. Figure \ref{fig:511_2} depicts the completion rate of each condition, compared to their self-reported prior experience. The effect of curriculum was found to be significant, i.e. the completion rates were not equally distributed amongst the 3 conditions ($\chi^2 (2, N=230) = 9.24, p < 0.01$).

Lastly, the duration per attempt for groups under PERM ($\mu = 61.02, \sigma = 66.41$) were significantly longer than that of the random curricula ($\mu = 45.01, \sigma = 19.68, p < 0.01$) and control condition ($\mu = 29.86, \sigma = 16.42, p < 0.01$). The average duration is plotted in Figure \ref{fig:511_3}.

\subsubsection{Discussion}
Collectively, these findings suggest that students trained with PERM were not only more likely to succeed on the test but also required fewer attempts to do so. Crucially, this positive impact of PERM on students remains consistent across individuals with diverse levels of prior experience with similar games. This consistency underscores the effectiveness of the adaptive curriculum implemented by PERM, demonstrating its capacity to benefit participants regardless of their varied backgrounds.

We were surprised that students under PERM had took significantly longer per attempt to complete the test. This observation hints at distinct behavioral differences among the learners, especially those exposed to higher difficulty levels. It's worth highlighting that participants were not explicitly informed that their performance was being evaluated based on the speed of level completion. This absence of explicit information could have influenced the more deliberate approach adopted by students exposed to the PERM framework.

\subsection{R2: Characteristics of Generated Levels}
This analysis investigates the levels generated by both the random curriculum and PERM, and how they evolve over time with the participants leading into the test. 

\subsubsection{Method}
\begin{figure*}[!t]
    \includegraphics[width=\linewidth,height=5cm]{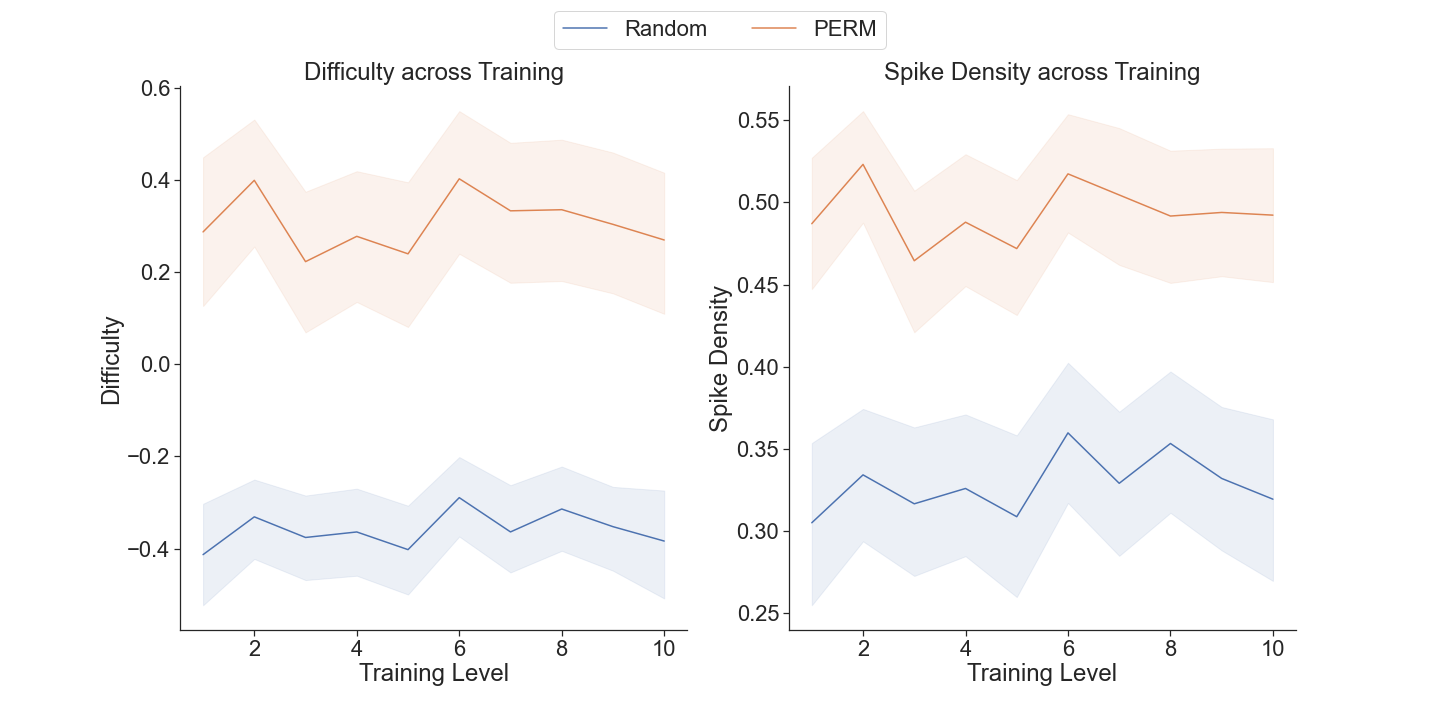}
    \caption{Participants trained by PERM (orange) were exposed to more difficult environments. Left: Difficulty of training levels, estimated by PERM. Right: Spike Density of the levels generated.}
    \label{fig:521}
\end{figure*}
During the training phase, we collected data on the types of level generated. We then used PERM to infer the difficulties of the environments generated, and plotted the average difficulty for each group over the course of the training. We also present the average \textit{spike density} across the training to illustrate its relationship with difficulty. 

\subsubsection{Results}
\begin{figure*}[!t]
    \includegraphics[width=\linewidth,height=5cm]{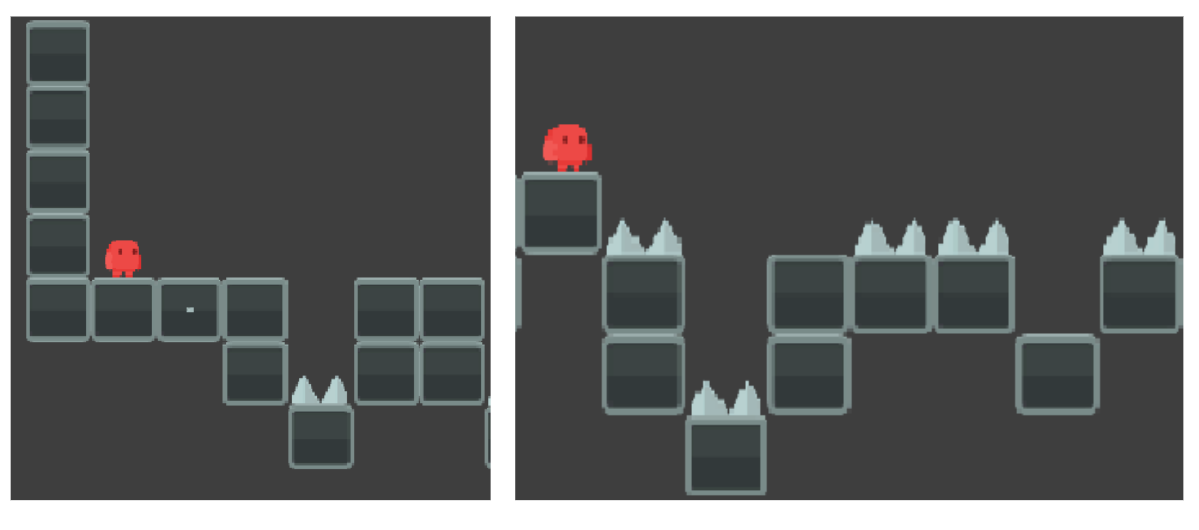}
    \caption{Possible segments of levels generated by PERM. The easy level (left) has lesser spikes and lesser variation in the terrain. In contrast, players have to navigate uneven terrains and jump across more spikes in the difficult level (right).}
    \label{fig:levelgen}
\end{figure*}
From Figure \ref{fig:521}, we found that participants under PERM, on average, were consistently exposed to environments of higher spike density, and consequently higher difficulty levels than those in the random curricula. We present an example of an ``easy'' and ``difficult'' environment generated by PERM in Figure \ref{fig:levelgen}. 

\subsubsection{Discussion}
According to PERM's assessment, the levels generated during PERM's training consistently exhibited higher difficulty levels compared to the random curricula. These findings indicate that students in the PERM condition consistently faced challenging scenarios and were exposed to greater levels of difficulty throughout their training. Consequently, this rigorous training appeared to be beneficial, as it translated well when students encountered an unseen and intricate level during the final test. It is important to emphasize that this final test was designed to be non-linear, requiring participants to execute more complex maneuvers in order to reach the goal. However, the exposure to and training under more challenging conditions in the PERM framework seemed to equip participants with the skills needed to perform well in this demanding test scenario.

Contrary to our initial expectations of logarithmic training curves with gradual growth followed by plateauing, participants in the PERM condition encountered challenging levels early, leading to a performance ceiling. This limitation arises from the inherent simplicity of the Jumper domain, which allows for swift comprehension of game mechanics and an upper threshold of difficulty that many participants under PERM reached. 

Nevertheless, we were encouraged to find that PERM was able to quickly infer the ability levels of learners, and present our participants with challenging levels early in their training. This is as opposed to the random curriculum, which would have generated trivial levels irrespective of student ability, thus wasting a training opportunity.

Finally, we conducted a visual inspection of the game levels (Figure \ref{fig:levelgen}) to confirm that PERM's characterization of 'easy' and 'difficult' aligns with our original design intentions for Jumper. This examination revealed that 'easy' levels exhibited a flatter terrain with fewer spikes, while 'difficult' levels presented players with an abundance of spikes and challenging, uneven pathways.

\section{Conclusion and Future Work}
We adapted PERM, an IRT-based generative model, to a 2D game and demonstrated its capabilities in training both RL Agents and human learners. For human learners, we conducted a user study to evaluate the effects of PERM-guided training in the Jumper environment over a random curricula, or no training at all. We are pleased to share the positive results observed for the PERM curriculum. 

In addition, we present a process in which RL could be used to collect data and subsequently train PERM to make inferences on level difficulty and student ability. Thereafter, the trained PERM could be deployed  as a teacher algorithm to provide adaptive training to virtually any student, in any parameterized environment. This work we demonstrated has massive potential to be generalized to other domains, such as complex video games as well as high school education. We present this work as a step towards leveraging recent AI techniques to support the establishment of new training systems that can be applied to real-world agents.
\medskip

{
\small
\bibliographystyle{abbrvnat}
\bibliography{references}
}
\appendix
\include{appendix}


\end{document}

%% file: appendix.tex
\section{Jumper Environment}
\label{app:env}
\begin{figure}[!ht]
    \centering
    \includegraphics[width=\linewidth]{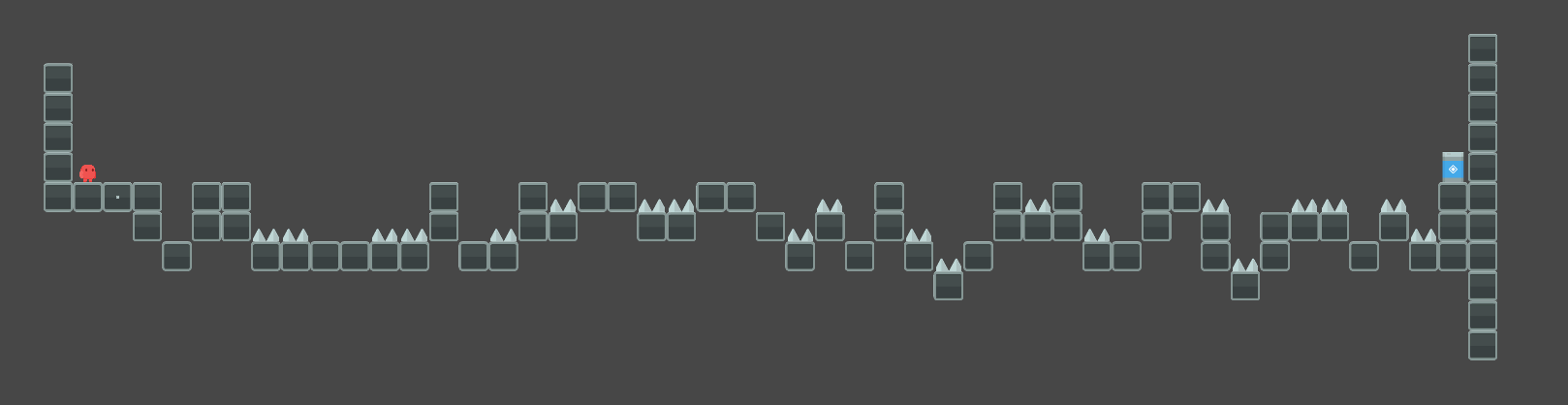}
    \caption{A sample of a Jumper level generated. Players would have to control the red character to jump across the spiked path to reach the blue goal at the far right of the level.}
    \label{fig:level}
\end{figure}
\subsection{Generation Details}
\label{appendix:generation}
This section details the Jumper environment, created by Unity \citep{juliani2020}. For training levels, teacher algorithms are empowered to control the parameters that will be used to generate the levels. We provide the description of the parameters below:
\begin{itemize}
    \item \textit{spike density} - the number of spikes in a given level. Teacher algorithms are to select a single scalar ranging from $[0.0,1.0]$, representing the probability that a spike would be generated for a given tile.  
    
    \item \textit{height variance} -  the degree to which the terrain varies. This is represented by a vector of size 4, with each scalar denoting the probability associated with the generation of a tile at a specific height. All values in the vector will sum to 1.0. In this context, the various tile heights correspond to distinct values, namely -1, 0, 1, and 2, which directly correspond to the respective y-coordinates of the tiles within the 2D environment.
    
\end{itemize}

Once the parameters have been determined, our Unity environment undergoes a sampling procedure. This procedure involves drawing samples from both a Bernoulli distribution and a Categorical Distribution, with these distributions being parameterized by the variables \textit{spike density} and \textit{height variance}, respectively. This sampling process is iteratively applied to each tile within the level, resulting in a total of 48 iterations.

\subsection{Final Test Level}
\includegraphics[width=\linewidth]{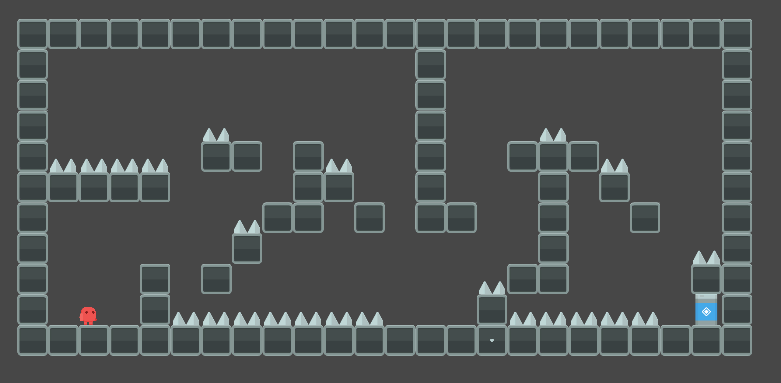}
    \label{fig:testlevel} \\
The above image is the Jumper test level that all participants were exposed to at the end of the training. Note that our algorithm highlighted in \ref{appendix:generation} is not able to produce levels of this complexity, and thus represents a level that is out-of-distribution with respect to the training set. 

\subsection{Limitations}
\label{sec:limitations}

The simplicity of the Jumper domain, intentionally designed for broad applicability to non-expert audiences and ease of verification, led to a ceiling effect in level difficulties. This deliberate choice allowed us to showcase the generalizability of our approach to both artificial and human learners, representing a step towards a general teacher algorithm for any parameterized domain. Jumper serves as an initial step towards generative models in adaptive training systems, and our future work aims to replicate these positive results in more complex domains, such as high school syllabus or commercial video games.

\section{PERM Proof}
We state and prove the revised PERM objective based on Variational Inference in the following theorem. We use notation consistent with the Variational Inference literature, and refer the motivated reader to \cite{kingma2013auto} for further reading.
\begin{theorem}
    Let $a$ be the ability for any student, and $d$ be the difficulty of any environment parameterized by $\lambda$. Let $r$ be the continuous response from the student on the environment. If we define the PERM objective as 

    \begin{align} \label{eq:PERM}
        \mathcal{L}_{PERM} & \triangleq  \mathcal{L}_{recon_r} + \mathcal{L}_{recon_{\lambda}} + \mathcal{L}_{A} + \mathcal{L}_{D} 
    \intertext{and assume the joint posterior factorizes as follows:} \\
        q_{\phi}(a, d | r, \lambda) &= q_{\phi}(a|d, r, \lambda) q_{\phi}(d| r, \lambda) 
    \end{align}
    then $\log p(r) + \log(\lambda) \geq \mathcal{L_{PERM}}$. $\mathcal{L_{PERM}}$ is a lower bound of the log marginal probability of a response $r$.

\end{theorem}

\begin{proof}
    Expand the marginal and apply Jensen's inequality:
    \begin{align*}
            \log p_{\theta}(r) + \log p_{\theta}(\lambda) \ge \nonumber 
             \mathbb{E}_{q_{\phi}(a, d | r)} [\log \frac{p_{\theta}(r,a,d, \lambda)}{q_{\phi}(a, d | r, \lambda)}] \\
            = \mathbb{E}_{q_{\phi}(a, d | r}) [\log {p_{\theta}(r| a,d)}] \\
            + \mathbb{E}_{q_{\phi}(a, d | r)} [\log {p_{\theta}(\lambda| d)}] \\
            + \mathbb{E}_{q_{\phi}(a, d | r)} [\log \frac{p_(a)}{q_{\phi}(a| d , r, \lambda)}] \\
            + \mathbb{E}_{q_{\phi}(a, d | r)} [\log \frac{p_(d)}{q_{\phi}(d | r, \lambda)}] \\
            = \mathcal{L}_{recon_r} + \mathcal{L}_{recon_{\lambda}} + \mathcal{L}_{A} + \mathcal{L}_{D}
    \end{align*}
Since $\mathcal{L}_{PERM}= \mathcal{L}_{recon_r} + \mathcal{L}_{recon_{\lambda}} + \mathcal{L}_{A} + \mathcal{L}_{D}$ and KL divergences are non-negative, we have shown that $\mathcal{L}_{PERM}$ is a lower bound on $\log p_{\theta}(r) + \log p_{\theta}(\lambda)$.

For easy reparameterization, all distributions $q_{\phi}(.|.)$ are defined as Normal distributions with diagonal covariance. 
\end{proof}

\section{Ability levels Across Time}
We investigate if abilities estimated by PERM is indicative of final performance on the test. 
\subsection{Method}
\begin{figure}[!h]
    \centering
    \includegraphics[keepaspectratio=true,scale=0.25]{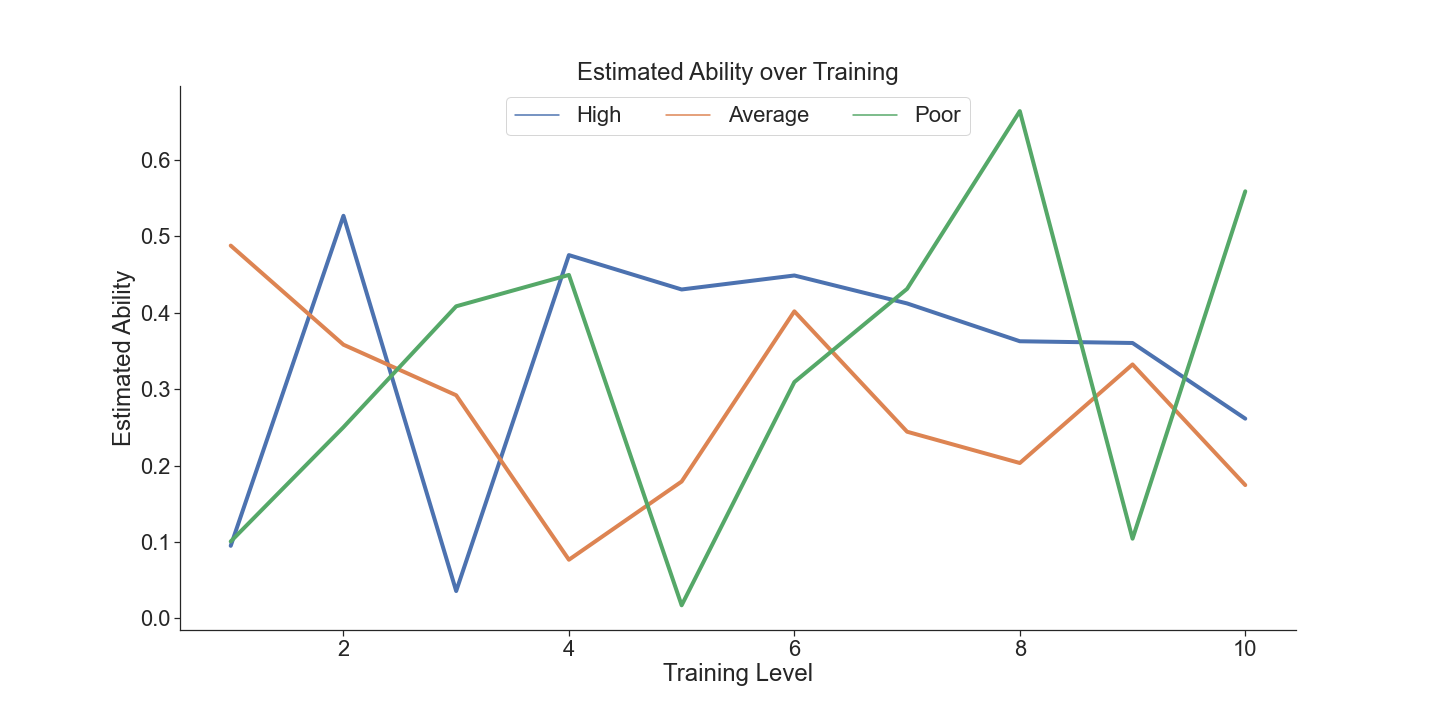}
    \caption{Abilities estimated by PERM was predictive of performance in the final test: High performers were estimated to have higher ability than average-performing participants in the later half of the training. Ability estimates for poor-performing participants seemed to have fluctuations.}
    \label{fig:531}
\end{figure}
We first categorize the top 25\%, middle 50\%, and bottom 25\% participants into high, average, poor performers based on their test performacne, and plot their average ability inferred by PERM across the training. 

\subsection{Results}
Figure \ref{fig:531} provides the ability levels of the participants under the PERM condition, across the training phase. Error bars are excluded for clarity. We observe that high performers tackled more difficult levels than those average performers, especially in the later half of the training. In contrast, poor performers had large fluctuations in the difficulty of their assigned levels, even occasionally assigned more difficult levels than the high performers. 

\subsection{Discussion}
This results suggest that the ability levels inferred by PERM can be predictive of the final performance and, by extension, the true ability of the student in real time, for average-high performers. We were further encouraged by PERM's capabilities in tracking student's ability, with few observations. 

Participants with poor performance showed significant ability level fluctuations due to PERM's reliance on the most recent performance, leading to an overestimation of their abilities in easy levels, and subsequent recommendation of challenging levels.